\documentclass[11pt,a4paper]{article}
\usepackage[hyperref]{emnlp2020}
\usepackage{times,latexsym}
\usepackage{url}
\usepackage[T1]{fontenc}
\usepackage{graphicx}
\usepackage{algorithm}
\usepackage{algorithmic}
\usepackage[leftcaption]{sidecap}
\sidecaptionvpos{figure}{t}
\usepackage{amsmath}
\usepackage{amsfonts}
\usepackage{amsthm}
\usepackage{todonotes}
\newcommand{\sref}[1]{Sec.\ \ref{#1}}    
\newcommand{\algref}[1]{Algorithm \ref{#1}} 
\newcommand{\prg}[1]{\noindent\textbf{#1}} 
\newcommand{\KL}[2]{\ensuremath{KL\left({#1}\, \| \, {#2}\right)}}
\usepackage{amssymb}
 \usepackage{multirow}
 \usepackage{dblfloatfix}
 \usepackage[normalem]{ulem}

\usepackage{tabularx}

\newtheorem{lemma}{Lemma}

\definecolor{Green}{RGB}{50,200,50}

\usepackage{microtype}

\aclfinalcopy 

\title{Continual Adaptation for Efficient Machine Communication}

\author{Robert D. Hawkins$^1$, Minae Kwon$^2$, Dorsa Sadigh$^{2,3}$, Noah D. Goodman$^{1,2}$ \\
  Departments of $^1$Psychology, $^2$Computer Science, and $^3$Electrical Engineering \\
  Stanford University, Stanford, CA 94305 \\
 \texttt{\{rxdh, mnkwon, ngoodman\}@stanford.edu}, \texttt{dorsa@cs.stanford.edu} \\}

\date{}

\begin{document}
\maketitle
\begin{abstract}
To communicate with new partners in new contexts, humans rapidly form new linguistic conventions.
Recent neural language models are able to comprehend and produce the existing conventions present in their training data, but are not able to flexibly and interactively adapt those conventions on the fly as humans do.
We introduce an interactive repeated reference task as a benchmark for models of adaptation in communication and propose a regularized continual learning framework that allows an artificial agent initialized with a generic language model to more accurately and efficiently communicate with a partner over time.
We evaluate this framework through simulations on COCO and in real-time reference game experiments with human partners.
\end{abstract}

\section{Introduction}

Communication depends on shared conventions about the meanings of words \cite{Lewis69_Convention}, but the real-world demands of language use often require agents to go \emph{beyond} fixed conventional meanings  \cite{Grice75_LogicConversation,Davidson86_DerangementOfEpitaphs}.
Recent work on \emph{pragmatic} and \emph{context-aware} models has approached this problem by equipping speaker and listener agents with the ability to explicitly reason about one another. 
Pragmatic reasoning allows listeners to infer richer intended meanings by considering counterfactual alternatives, and allows speakers to be appropriately informative, not merely truthful \cite{GoodmanFrank16_RSATiCS,andreas2016reasoning,fried2017unified,MonroeEtAl17_ColorsInContext,VedantamEtAl17_ContextAwareCaptions}.

\begin{figure}[t!]
\includegraphics[scale=.6]{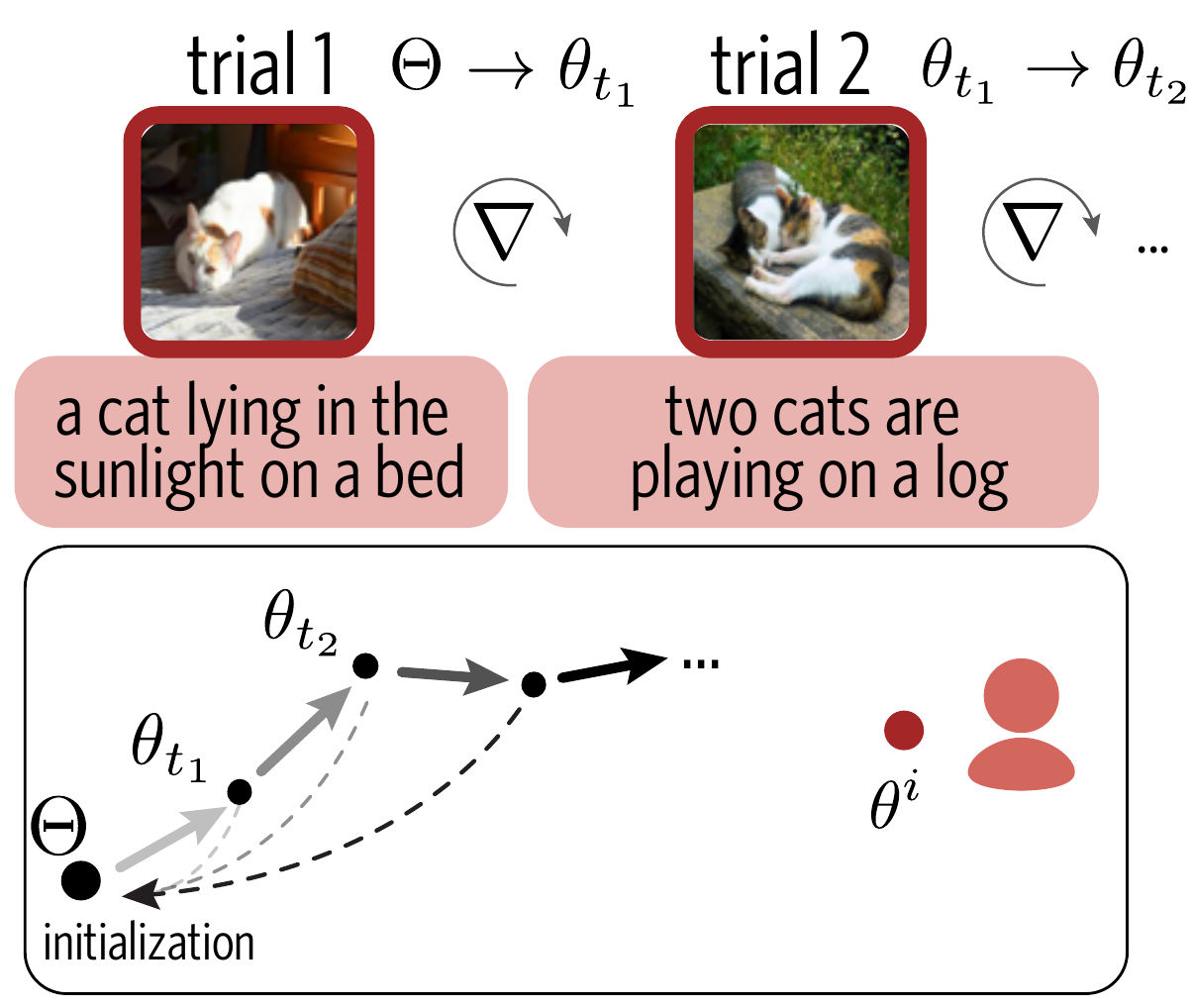}
\caption{We introduce a regularized continual learning approach allowing agents initialized with a pre-trained language model $\Theta$ to iteratively infer the language model $\theta^i$ used by a partner, over repeated interactions $\{t_1, t_2\dots\}$ in an online reference game.}
\vspace{-1em}
\label{fig:refgame}
\end{figure}

These models have largely focused on one-shot settings, where the context is the immediate visual environment.
In common interactive settings, however, the relevant context for pragmatic competence also includes the history of previous interactions with the same communication partner.
Human interlocutors are able to establish \emph{ad hoc} conventions based on this history \cite{ClarkWilkesGibbs86_ReferringCollaborative,clark1996using}, allowing for increasingly \emph{accurate} and \emph{efficient} communication. 
Speakers can remain understandable while expending significantly fewer words \cite{KraussWeinheimer64_ReferencePhrases,orita2015discourse,staliunaite_getting_2018, hawkins2019characterizing, stewart_characterizing_2020}. 

For example, consider a nurse visiting a bed-ridden patient at their home.
The first time the patient asks the nurse to retrieve a particular medication, they must painstakingly identify a specific bottle, e.g.~``the medicine for my back pain in a small blue medicine bottle labeled Flexeril in my bathroom.''
But after a week of care, they may just ask for the ``back meds'' and expect the nurse to know which bottle they mean.
Such flexibility poses a challenge for current pragmatic models. 
For an artificial agent to establish new conventions, as humans do, it must go beyond pragmatic reasoning at the single-utterance timescale to \emph{learn} about its partners over longer timescales. 

Here, we propose that the problem of \emph{ad hoc} convention formation can be usefully re-formulated as an inference problem amenable to online domain adaptation. 
Our approach is motivated by a growing body of evidence in cognitive science that humans quickly re-calibrate their expectations about how language is used by different partners \cite{grodner201110,Yildirim16_TalkerSpecificityQuantifiers}.
This empirical work highlights three key challenges facing a scalable adaptation approach.
First, because the target data comes from intentional agents, \emph{pragmatic reasoning} must be deployed throughout adaptation to strengthen inferences \cite{FrankGoodmanTenenbaum09_Wurwur}.
Second, because the data is sparse, strong adaptation risks catastrophic forgetting; yet, human speakers are able to revert to their background expectations for the next interlocutor \cite{WilkesGibbsClark92_CoordinatingBeliefs,MetzingBrennan03_PartnerSpecificPacts}. 
Third, the ability to ground the meanings of later, shorter utterances (e.g. ``back meds'') in the use of earlier, longer utterances requires a \emph{compositional} representation; otherwise the connection between the utterances is not clear \cite{hawkins2019characterizing}.

Our primary contribution is an online continual learning framework for transforming pragmatic agents into \emph{adaptive} agents that can be deployed in real-time interactions.
This framework is shown schematically in Fig.~\ref{fig:refgame}: after each trial, we take a small number of gradient steps to update beliefs about the language model used by the current partner.
To evaluate our framework, we first introduce a benchmark \emph{interactive} repeated reference task (Fig.~\ref{fig:task}) using contexts of natural images.
In \sref{sec:approach}, we introduce the three core components of our algorithm: (i) a contrastive loss objective incorporating explicit pragmatic reasoning, (ii) a KL regularization objective to prevent overfitting or catastrophic forgetting, and (iii) a data augmentation step for compositionally assigning credit to sub-utterances. 
In \sref{sec:evaluation}, we report experiments demonstrating that this algorithm enables more effective communication with naive human partners over repeated interactions.
Finally, in \sref{sec:analysis} we report a series of ablation studies showing that each component plays a necessary role, and close with a discussion of important areas for future research in \sref{sec:discussion}

\section{Related work}

\paragraph{Personalizing language models.}
Adapting or personalizing language models is a classic problem of practical interest for NLP, where shifts in the data distribution are often found across test contexts \cite{kneser1993dynamic,riccardi2000stochastic,bellegarda2004statistical,ben2010theory}. 
Our approach draws upon the idea of dynamically fine-tuning RNNs \cite{mikolov2010recurrent,krause2017dynamic}, which has successfully explained key patterns of human behavior in self-paced reading tasks \cite{van2018neural}.
We also draw on the regularization objectives proposed in this literatures \cite{li2007bayesian,liu2016investigations}. 
However, the interactive communicative setting we consider poses several distinct challenges from traditional speech recognition \cite{miao2015speaker} or text classification settings  \cite{blitzer2007biographies,glorot2011domain} for which adaptation is typically considered.
Partner-specific observations of language use are sparser, must be incorporated online, and are generated by intentional agents.

\paragraph{Incorporating discourse history.}

Previous work has incorporated discourse history in reference games using explicit co-reference detection \cite{roy2019leveraging} or contribution tracking \cite{devault_learning_2009} techniques. 
An alternative approach is to include embeddings of the history as conditional input to the model at test time \cite{haber2019photobook}.
Similar approaches have been proposed for sequential visual question answering \cite{ohsugi2019simple,choi2018quac}.
Rather than pre-training a fixed, monolithic language model and incorporating shared history on top of this model at test time, we suggest that the underlying language model itself ought to be continually adapted over the course of an interaction.

\paragraph{Bayesian models of adaptation.}

Models of adaptation in cognitive science are typically formulated in terms of (hierarchical) Bayesian belief-updating based on evidence of language use \cite{KleinschmidtJaeger15_RobustSpeechPerception,roettger2019evidential,delaney2019neural,HawkinsFrankGoodman17_ConventionFormation,SchusterDegen}.
In these models, each new observation is taken as statistical evidence about the partner's language model, allowing pairs to coordinate on shared expectations and ground new conventions in their partner's previous behavior (see Sec. \ref{sec:bayes}). 
While these models capture key theoretical properties of human adaptation, they do not scale well to natural-language applications, where neural networks are dominant.

\section{Approach}
\begin{figure}[t]
\includegraphics[scale=.6]{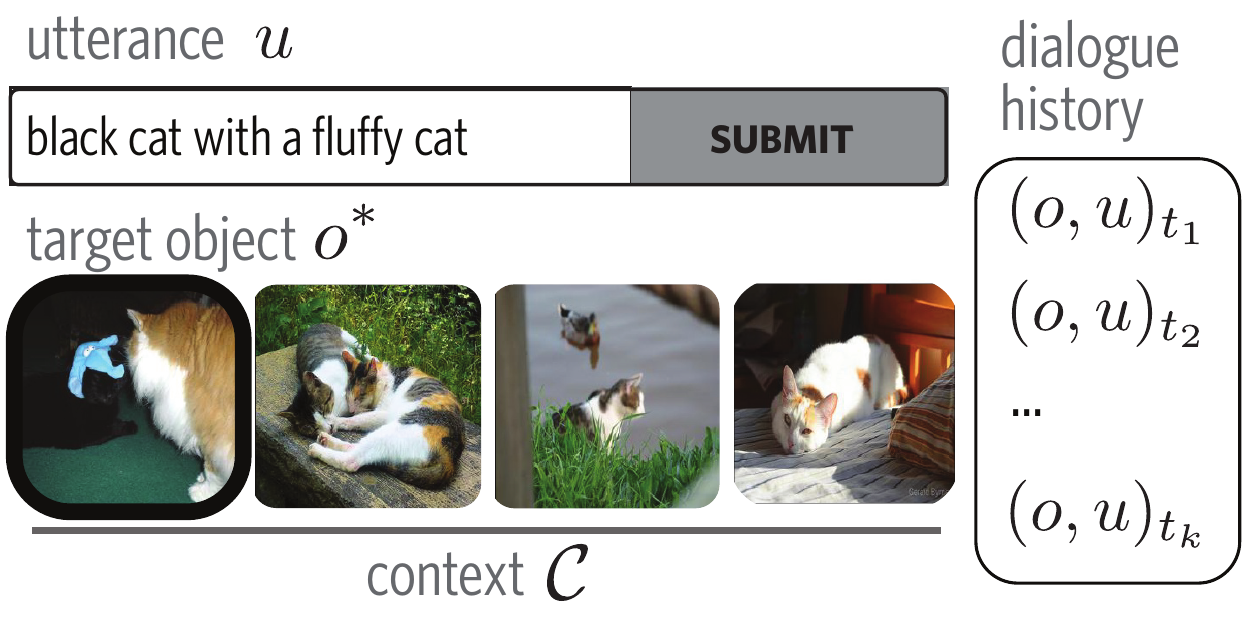}
\caption{In a repeated reference game, a speaker agent must repeatedly communicate the identity of the same objects in context to a listener agent.}
\label{fig:task}
\end{figure}

\label{sec:approach}
We begin by recasting convention formation as an online domain adaptation problem. 
As in previous computational approaches to pragmatics \cite[e.g.][]{GoodmanFrank16_RSATiCS,andreas2016reasoning}, we formulate this problem as an inference about another agent. 
The key theoretical idea is to expand the scope of pragmatic inference from the single-utterance timescale to evidence accumulated over longer timescales of an interaction. 
In addition to inferring a partner's intended meaning (or interpretation) for each individual utterance, an adaptive agent pools across previous utterances to infer the distinct but stable way their partner uses language.
Under this inference framework, an agent must both (1) begin with background expectations about language shared across many partners, and (2) have a mechanism to rapidly learn the specific language model used by the current partner.
Our work assumes a conventional neural language model as the starting point and focuses on the partner-specific inference problem.
In this section, we describe our repeated reference game benchmark task (\ref{sec:task}), review the underlying problem as it has been previously formulated in a Bayesian framework (\ref{sec:bayes}), and finally describe our algorithm for adapting neural language models (\ref{sec:neural}).

\subsection{Repeated reference game task}
\label{sec:task}

As a benchmark for studying domain adaptation in communication, we use the \emph{repeated reference game} task (Fig.\ \ref{fig:task}), which has been widely used in cognitive science to study partner-specific adaptation in communication \cite{KraussWeinheimer64_ReferencePhrases,ClarkWilkesGibbs86_ReferringCollaborative,WilkesGibbsClark92_CoordinatingBeliefs}.
In this task, a speaker agent and a listener agent are shown a context of images, $\mathcal{C}$ (e.g. four images of cats). 
On each trial, one of these images is privately designated as the \emph{target object}, $o^*$, for the speaker (e.g. the image with the thick border shown on the left).
The speaker agent thus takes the pair $(o^*, \mathcal{C})$ as input and returns an utterance $u$ (e.g. ``black cat with a fluffy cat'') that will allow the listener to select the target from $\mathcal{C}$.
The listener agent takes $(u, \mathcal{C})$ as input and returns a softmax probability for each image, which it uses to make a selection.
Both agents then receive feedback about the listener's selection and the identity of the target. 
Critically, the sequence of trials is constructed so that each image appears as the target several times.
For example, our evaluations loop through each target six times, allowing us to observe how communication about each image changes as a function of dialogue history (see Fig.\ S1 in Supplementary Materials for examples).


\subsection{The inference problem}
\label{sec:bayes}

We begin by assuming that agents represent the semantics of their language as a function relating natural language utterances $u$ to actual states of the world $o$ (here, images). 
We further assume that this function belongs to a family parameterized by $\theta$, and denote the parameter used by a particular agent $i$ with $\theta^i$ (see Fig.~\ref{fig:refgame}).
If an artificial agent knows the true value of $\theta^i$ -- their current partner's semantics\footnote{Traditionally, this semantic function is truth-conditional, mapping utterance-state pairs to Boolean values, but recent approaches have shifted to more graded, real-valued functions such as those implemented by neural networks.} -- they are in a better position to understand them, and to be understood in turn.
However, because $\theta^i$ is not directly observable and $\theta$ varies across partners and contexts, it must be inferred \cite{BergenLevyGoodman16_LexicalUncertainty}.
Furthermore, it is in the agent's best interest to use its updated beliefs about its partner's $\theta$ to guide its own production and interpretation.
An important consequence of this formulation is that conventionalization, the process by which parties converge on an efficient way to refer to something, emerges naturally as a consequence of mutual adaptation, the process by which each party independently tries to infer their interlocutor's language model \cite{SmithGoodmanFrank13_RecursivePragmaticReasoningNIPS,hawkins2020generalizing}.

This is the central computational problem of adaptation, which we formalize as follows.
Following Bayes Rule, the adaptive agent's beliefs about $\theta^i$, conditioning on observations $D^i$ from the shared history of interactions in that context, are:
\begin{equation}
	P(\theta^i | D^i, \Theta)  \propto P(D^i | \theta^i) P(\theta^i | \Theta)
	\label{eq:lexicon_update}
\end{equation}
This formulation decomposes the inference into two terms, a prior term $ P(\theta^i | \Theta)$ and a likelihood term $P(D^i | \theta^i)$.\footnote{For the rest of this paper, we only consider the case of adapting to one partner, so we will drop the partner index $i$.}
The prior captures the idea that different partners share some general features of the semantics, represented by $\Theta$, since they speak the same language; in the absence of partner-specific information, the agent ought to be regularized toward this background knowledge. 

The likelihood term, on the other hand, accounts for direct evidence of language use.
It represents an explicit forward model of an agent: different latent values of $\theta$ generate different observable actions.
In other words, the standard single-utterance pragmatic inference problem is nested within the longer-timescale inference about $\theta$. 
While explicit reasoning about the other agent is typically considered at the time of action selection (i.e. when the speaker is choosing an utterance, or when the listener is choosing a referent; \citealp{GoodmanFrank16_RSATiCS,andreas2016reasoning}), this likelihood term importantly incorporates such reasoning at the time of \emph{adaptation} (i.e. when updating beliefs about $\theta$ based on previous actions; \citealp{FrankGoodmanTenenbaum09_Wurwur, SmithGoodmanFrank13_RecursivePragmaticReasoningNIPS}).


\subsection{Continual adaptation for neural models}
\label{sec:neural}

If we let $\theta$ be the weights of an image-captioning network, then the background knowledge shared across partners, $\Theta$, corresponds to a pre-trained initialization, and conditioning on partner-specific data under a Bayesian prior corresponds to regularized gradient descent on $\theta$.
We exploit this connection to derive an online continual learning scheme that addresses the challenges of adapting to a human partner in a repeated reference game task. 


\paragraph{Architecture and algorithm overview.}
Concretely, we consider an architecture that combines a convolutional visual encoder (ResNet-152) with an LSTM decoder \cite{vinyals_show_2015}.
The LSTM takes a 300-dimensional embedding as input for each word in an utterance and its output is linearly projected back to a softmax distribution over the vocabulary size.
To pass the visual feature vector computed by the encoder into the decoder, the final layer of ResNet was replaced by a fully-connected adapter layer.
This layer was jointly pre-trained with the decoder on the COCO training corpus  \cite{lin2014microsoft} and frozen. 
The COCO corpus contains images of common objects, each annotated with multiple human captions. 
The CNN-LSTM architecture allows an agent to select utterances, by using beam search over captions given a target image as input, and also to select objects from the context, by evaluating the likelihood of the caption for each image in context and taking the most likely one. 

Critically, we assume the agent will select actions on each trial using the value of $\theta$ it believes its partner to be using, so updating its own model is equivalent to updating expectations about its partner's model.
Using the pre-trained model as our initialization, we can fine-tune the decoder weights (i.e. word embeddings, LSTM, and linear output layer) within a particular communicative interaction.
Our algorithm is specified in \algref{alg:listener_update}.
Upon observing the utterance-object pair produced on each trial of the repeated reference game (Line 3), we take a small number of gradient steps updating the model weights to reflect the usage observed so far (Lines 4-7).
Our adaptation objective function (Line 6) is built from combining a standard cross-entropy term with a KL-based regularization term to prevent catastrophic forgetting and a contrastive term to incorporate pragmatic reasoning about the visual context. 
In the following sections, we explain these terms and also introduce a final component of our approach: compositional data augmentation.

\begin{algorithm}[t!]
	\caption{Update step for adaptive model}
    \label{alg:listener_update}
\begin{algorithmic}[1]
    \STATE {\textbf{Input}: $\theta_t$: weights at time $t$}
    \STATE {\textbf{Output}: $\theta_{t+1}$: updated weights}
    \STATE {\textbf{Data}: ($u_t, o_t$): observed utterance and object}
    \FOR{step}
        \STATE sample augmented batch $\mathbf{u} \sim \mathcal{P}(u_t)$\;
        \STATE let $f_{\theta_t} =  \log P_{\theta_t}(\mathbf{u} | o_t)  + \log P_{\theta_t}(o_t | \mathbf{u}) - \text{\hspace{1.65cm}}\text{reg}(o_{1:t-1},u_{1:t-1})$
        \STATE update $\theta_{t} \leftarrow \theta_t + \beta\nabla f_{\theta_t}$
   \ENDFOR
\end{algorithmic}
\end{algorithm}

\prg{Utterance likelihood.} 
For our benchmark repeated reference game, the data obtained on trial $t$ is a paired observation of an utterance $u$ and an intended object of reference $o$. 
The simplest learning objective for $\theta$ is the standard cross-entropy loss: the likelihood of this utterance being produced to convey the intended target in isolation: $P_\theta(u | o)$.
This likelihood can be computed directly from the neural captioning model, where the probability of each word in $u=\{w_0, \dots, w_\ell\}$ is given by the softmax decoder output conditioned on the sentence so far, $P_{\theta_t}(w_i |o, w_{-i})$, so:
\begin{equation}
	P_{\theta_t}(u | o) \propto \prod_{i<\ell} P_{\theta_t}(w_i |o, w_{-i})
	\label{eq:speaker}
\end{equation}
\prg{Contrastive likelihood.} 
The same object-utterance pairs can be viewed as being generated by a listener agent selecting $o$ \emph{relative to the other distractors} in the immediate context $\mathcal{C}$ of other objects.
This reasoning requires inverting the captioning model to evaluate how well the utterance $u$ describes each object in $\mathcal{C}$, and then normalizing:
\begin{equation}
 P_{\theta_t}(o | u, \mathcal{C}, \theta_t) \propto 
   P_{\theta_t}(u | o)P(o) 
	\label{eq:listener}
\end{equation}
This inversion is based on models of one-shot pragmatic inference in reference games   \cite{GoodmanFrank16_RSATiCS,andreas2016reasoning,VedantamEtAl17_ContextAwareCaptions,cohn2018pragmatically}.
While optimizing the utterance likelihood serves to make the observed utterance more likely for the target in \emph{isolation}, optimizing the contrastive likelihood allows the agent to make a stronger inference that it does \emph{not} apply to the distractors.

\prg{KL Regularization.} 
Fine-tuning repeatedly on a small number of data points presents a clear risk of catastrophic forgetting \cite{robins_catastrophic_1995}, losing our ability to produce or understand utterances for other images. 
While limiting the number of gradient steps keeps the adapted model somewhat close to the prior, we will show that this is not sufficient (see \sref{sec:KL_analysis}).
Because small differences in weights can lead to large differences in behavior for neural models, we also consider a regularization that tethers the \emph{behavior} of the adapted model close to the behavior at initialization.
Specifically, we consider a \emph{KL regularization} term that explicitly minimizes the divergence between the captioning model's output probabilities before and after fine-tuning for unseen images \cite{yu2013kl, galashov_information_2018}.
It is not tractable to take the KL divergence over the (nearly infinite) space of all possible natural-language utterances. 
Hence, we approximate the divergence incrementally by expanding from the maximum a posteriori (MAP) word denoted $w^*$ at each step according to the initial model $P_{\Theta}$ (see Appendix A):
\begin{equation}
\sum_{i<\ell} 
  D_{\mathrm{KL}}
	\left[
	P_{\Theta}(w_i |o, w^*_{-i})
	\parallel
	P_{\theta_t}(w_i|o, w^*_{-i})
	\right]
\label{eq:speaker_reg}
\end{equation}
where $\ell$ is the length of the MAP caption.
This loss is then averaged across random images sampled from the full domain $\mathcal{O}$, not just those in context.


\prg{Compositional data augmentation.} 
Agents should be able to infer previous successes on a longer utterance (e.g. ``two men are sitting on a bench''), that the component parts of this utterance (e.g. ``two men'', ``a bench'') are also likely to convey the intended meaning.
In the absence of a (weakly) compositional representation, a speaker has no way of doing credit assignment: observing that a listener successfully chose the target upon hearing a long utterance only provides further evidence for the full utterance.
Fine-tuning an LSTM architecture will increase the likelihood of sub-strings to some extent after a successful selection, but this is insufficient for two reasons.
First, not all sub-strings are syntactically well-formed referring expressions (e.g. ``two men are''), and the LSTM lacks a syntactic representation to represent such coherence.
Second, the likelihood of the full utterance will always be increased by more than any sub-utterance.

To address these problems, we explored a data augmentation step that introduces a stronger compositionality bias via referential entailments \cite{young2014image}.
After each trial, we augmented the speaker's utterance $u$ with a small denotation graph $D(u)$ containing the set of all noun phrases found in the syntactic dependency parse of $u$, and optimize our objective function on batches of these entailments.
By independently updating expectations about well-formed entailments alongside the longer utterances that were actually produced, we hypothesized that our model could more naturally ground shorter, conventionalized labels in the shared history of successful understanding. 

\prg{Local rehearsal.}
A second form of augmentation we explore is \emph{local rehearsal}: at each step we include data from the history of interaction $D = \{(u, o)\}_{1:t}$  up to the current time $t$, to prevent overfitting to the most recent observation.
In practice we subsample batches from the interaction history in a separate loss term with its own weighting coefficient, ensuring the new data point and a batch of its subphrase augmentations are used in every gradient step. We initialize $D$ with the utterance the model generates for each object.

\begin{figure*}[t!]
\centering
\includegraphics[scale=1.05]{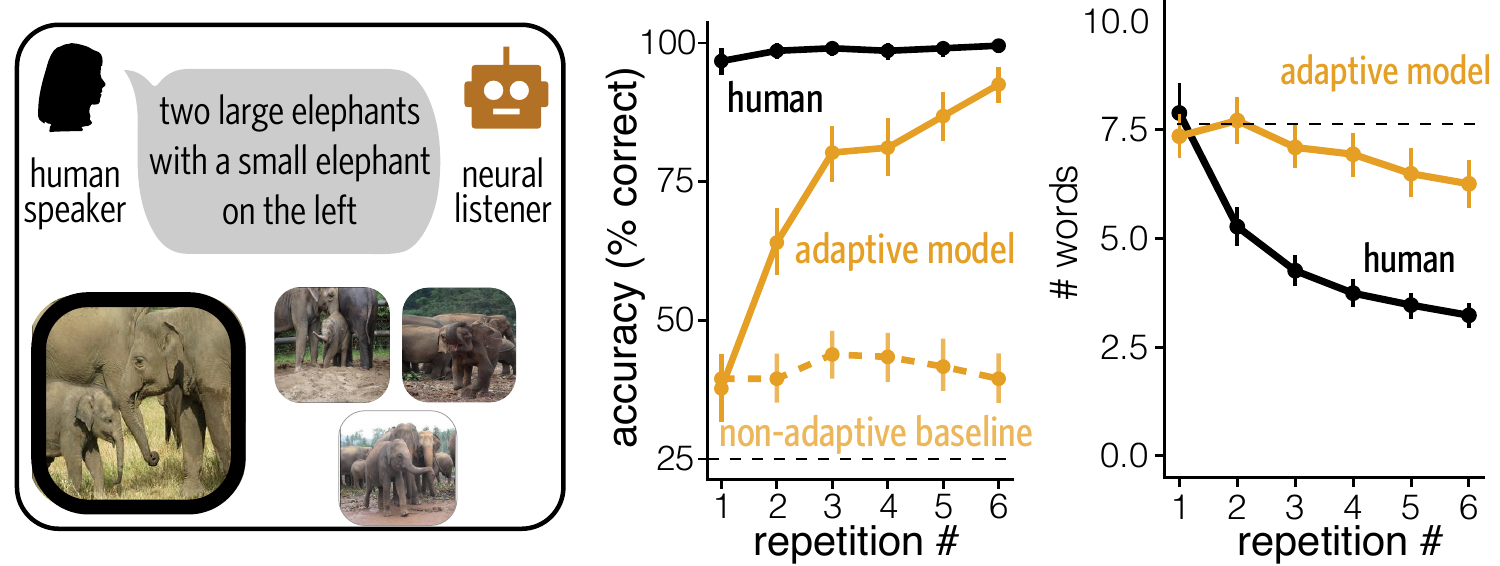}
\caption{Communication becomes more efficient and accurate as our model adapts to a human speaker. Example contexts and utterances are shown. Error bars are bootstrapped 95\% CIs.}
\label{fig:results}
\end{figure*}

\section{Interactive human evaluations}
\label{sec:evaluation}
In this section, we evaluate our model's performance  in \emph{real-time interactions} with human speakers. 
Our artificial agent was paired with human partners to play a repeated reference game using images from the validation set of the COCO corpus \cite{lin2014microsoft,chen2015microsoft} as the targets of reference.
Critically, we constructed contexts to create a diagnostic mismatch between the COCO pre-training regime and the referential test regime. 
Specifically, we chose contexts such that the model's \emph{accuracy} --- the probability of identifying the target --- would be poor at the outset. 

To obtain appropriately challenging contexts, we used our pre-trained model's own visual encoder to find sets of highly similar images within the same category.
We first extracted 256-dimensional feature vectors for each image from the final, fully-connected layer of the encoder. 
We then used these features to partition the images into 100 groups using a $k$-means algorithm, sampled one image from each cluster, and took its 3 nearest neighbors in feature space, yielding 100 unique contexts of 4 images each. 
This adversarial process explicitly identified contexts that our pre-trained captioning model would be poorly equipped to distinguish. 

\paragraph{Human baselines.}
We first investigated the baseline performance of human speakers and listeners. 
We recruited 108 participants (54 pairs) from Amazon Mechanical Turk and automatically paired them into an interactive environment with a chatbox.
For each pair, we sampled a context and constructed a sequence of 24 trials structured into 6 repetition blocks, where each of the 4 images appeared as the target once per block. 
We prevented the same target appearing twice in a row and scrambled the order of the images on each player's screen on each trial. 
We found that pairs of humans were highly accurate, with performance consistently near ceiling (Fig.\ \ref{fig:results}, black lines).
At the same time, their utterances grew increasingly efficient: their utterances reduced in length across repeated interaction ($t~=~25.8,~p~<~0.001$).\footnote{Note that our contexts were selected to be challenging under the impoverished language prior of our pre-trained listener model, but were not expected to require any adaptation for human listeners to achieve high accuracy; see \citet{hawkins2019characterizing} for a more challenging stimulus domain used to elicit strong human adaptation.}

\subsection{Model performance}
\label{sec:listening_task}
Next, we evaluated the performance of our adaptive model in the listener role (for a similar analysis of our model in the speaker role, see Appendix D).
We recruited 57 additional participants from Amazon Mechanical Turk who were told they would be paired with an artificial agent learning how they talk.
This task was identical to the one performed by pairs of humans, except we allowed only a single message to be sent through the chatbox on each trial. 
This message was sent to a server where the model weights from the previous trial were loaded to the GPU, used to generate a response, and updated for the next round.
The approximate latency for the model to respond was 5-10s depending on how many games were running simultaneously.

\begin{figure*}[t]
\centering
\includegraphics[scale=0.69]{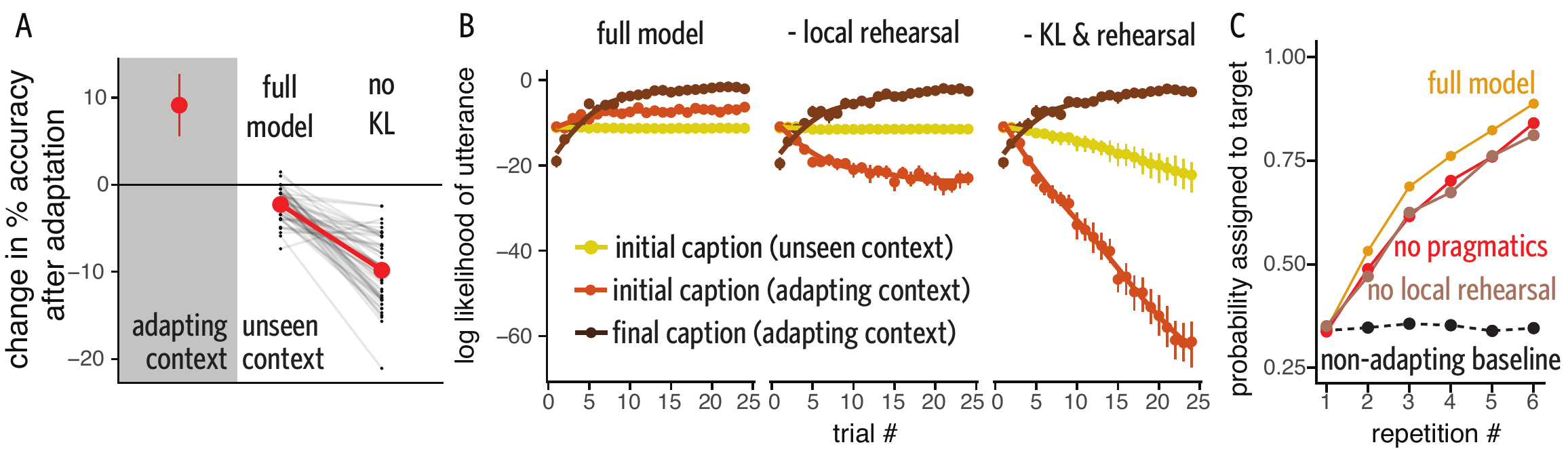}
\caption{Ablation studies of the listener model. (A-B) KL regularization prevents catastrophic forgetting over the course of adaptation. (C) Local rehearsal and pragmatic reasoning independently contribute to successful listener adaptation. Error bars are bootstrapped 95\% CIs.}
\label{fig:forgetting}
\end{figure*}

For our adaptation objective function, we used a linear combination of the utterance and contrastive losses and the KL-regularization (see Appendix B for hyper-parameter settings).
We also used local rehearsal and compositional data augmentation. 
While the pre-trained model initially performs much less accurately than humans, as expected, our adaptive listener shows rapid improvement in accuracy over the course of interaction (Fig.\ \ref{fig:results}).
In a mixed-effects logistic regression predicting trial-level accuracy, including pair- and image-level random effects, we found a significant increase in the probability of a correct response with successive repetitions, $z=12.6,~p <0.001$, from 37\% correct (slightly above chance levels of 25\%) to 93\% at the end.
To test whether this success can be attributed to the initial quality of the listener model, or to \emph{humans} adapting to a relatively unchanging model, we examined the performance of a non-adapting baseline (i.e. a model using the pre-trained model weights on every trial).
We evaluated this baseline offline, using the utterances we recorded from the online games. 
This baseline showed no improvement, staying only slightly above chance accuracy over the course of the task.

\section{Analysis}
\label{sec:analysis}

We now proceed to a series of ablation analyses that analyze the role played by each component of our approach.
These analyses involve offline simulations conducted on the data we collected in the previous section.

\subsection{KL regularization prevents catastrophic forgetting}
\label{sec:KL_analysis}

We begin by testing the effectiveness of our KL regularization term (Eq.~\ref{eq:speaker_reg}) for preventing catastrophic forgetting.
We reasoned that changing expectations in the adaptation context should not interfere with expectations in other, unseen contexts.
To directly analyze such interference, we adapted an ablated variant of our listener model over the course of a game with one context of images, and then measured its average accuracy identifying the target given the initial utterances produced by different speakers on different (unseen) contexts.
We then compared this test accuracy with the baseline accuracy achieved by an \emph{unadapted} listener model.

We cross-validated these estimates over many adaptation contexts.
Specifically, because the baseline was already close to chance on `challenging' contexts (Fig.\ \ref{fig:results}), we used an additional set of 52 human-human interactions we collected in easier contexts (where images belonged to different COCO categories) to better expose degradations in performance. 
While accuracy significantly increased compared to baseline in the adapting context for both variants, we found a 10\% drop in accuracy on unseen contexts for the ablated variant with no KL term, compared to only a 2\% drop in the model using the full loss ($t(51)=12.2, p <0.001$ in a paired $t$ test; see Fig.\ \ref{fig:forgetting}A).

Next, to more thoroughly probe the \emph{progression} of interference, we conducted a second analysis examining the likelihood assigned to different captions by the listener model over the course of adaptation.
We tracked both the initial captions produced by the pre-trained initialization in the adapting context and in unseen contexts. 
To obtain unseen contexts, we sampled a set of images from COCO that were not used in our experiment, and generated a caption for each.
We also generated initial captions for the \emph{target} objects in the adapting context. 
We recorded the likelihood of all of these sampled captions under the model at the beginning and at each step of adaptation until the final round. 
Finally, we greedily generated an utterance for each target at the \emph{end} and retrospectively evaluated its likelihood at earlier points during adaptation.

These three likelihood curves are compared for ablated models in Fig.\ \ref{fig:forgetting}B.
By definition, the final caption in the adapting context becomes more likely in all cases (brown line).
Without the local rehearsal mechanism, the initial caption the model expected in the adapting context becomes less likely as it is replaced by the human partner's preferred caption (red line).
Only when the KL term is removed, however, do we find interference with the model's expectations for unseen contexts (yellow line).
Thus, we find that KL regularization plays a critical role in preventing catastrophic forgetting.

\subsection{Pragmatics and local rehearsal improve listener performance}

Next, we consider the contributions of other key components for success.
Specifically, we constructed ablated variants of our model with no pragmatics (i.e. no contrastive loss term during adaption), and with no local rehearsal (i.e. no ability to keep training on batches from the history of the interaction).
We simulated adaptation for these ablated variants on the 57 games where human speakers produced utterances for our listener model, and examined the probability assigned to the target after hearing each utterance (Fig.\ \ref{fig:forgetting}C).
We found in a mixed-effects regression that each of these components independently contributes to success, as the ablated variants perform significantly worse than the full model (${z=2.1,~p = 0.03}$ and ${z=3.6,~p < 0.001}$ for variants with no local rehearsal and no pragmatics, respectively; see Appendix C for regression details). 
Compared to an entirely non-adapting baseline, however, even these ablated variants improved over time.

\section{Discussion}
\label{sec:discussion}

\paragraph{Relationship to human adaptation}

The theoretical ties between our approach and proposed cognitive models of human adaptation raises several questions. 
First, it is possible that improved performance could be driven by human \emph{speakers} adapting in response to our listener agent's successes and errors rather than the other way around. 
While some degree of human adaptation is inevitable -- for example, humans only seemed to shorten their utterances once our models' accuracy began to rise -- human adaptation alone is insufficient to explain gains in accuracy.
If these gains were due to human speakers gradually discovering utterances that a pre-trained (non-adapting) model could understand, we would expect some gains in the accuracy of our baseline non-adapting model over time. 
Furthermore, we found that the handful of human speakers that dramatically changed their descriptions across rounds actually performed worse than those who adhered to consistent descriptions.

With this said, the extent of adaptation in human-computer dialogue is known to be affected by human participants' expectations about the artificial agent \cite{branigan2011role,koulouri2016and}, potentially including expectations about whether it will be adaptive or not.
Bi-directional adaptation effects may be more pronounced in other dialogue settings where the human and model both speak, giving the human an opportunity to re-use utterances produced by the model.
It will be important for future work to evaluate non-adaptive baselines \emph{online} rather than offline, as we did, in order to observe exactly how humans respond to, or compensate for, non-adaptive agents.

Second, it is natural to ask how our model would perform in the \emph{speaker} role with a human \emph{listener}, using their (sparse) response success as feedback rather than their utterances.
In ongoing work, we have found that the same approach allows a (pragmatic) model to converge to more efficient conventions in the speaker role (see Appendix D in supplemental), such that the same language model can flexibly switch between speaker and listener roles with the same human partner. 
Still, it is unlikely that this speaker model reduces in the same way as human speakers do (see Supplemental Fig.~S1 for examples). 
Differences may reflect additional accessibility, grammaticality, or compositionality biases in humans; direct comparisons remain an open question for cognitive science. 

Third, scaling the principles of computational-level Bayesian cognitive models to neural networks capable of adapting to natural language in practice required several algorithmic-level innovations which are not yet plausible proposals for human cognition \cite{marr2010vision}. 
While our local rehearsal mechanism may be consistent with replay mechanisms in human memory, our KL regularization mechanism implausibly requires earlier parameter values of the model to be held in memory.
Our data augmentation mechanism was introduced specifically to compensate for the inability of the LSTM architecture to propagate the use of a referring expression to its entailments, but we expect that human language processing mechanisms achieve this effect by different means.
We expect further work to refine these algorithmic components as neural language models continue to advance.

\paragraph{Relationship to language learning}

Our work is also related to broader efforts to ground language learning and emergent communication in usage, where artificial agents are trained to use language \emph{from scratch} by playing interactive reference games \cite{wang2016learning,lazaridou2016multi,wang2017naturalizing,chevalier2018babyai}. 
Rather than starting our agents from scratch, we have emphasized the need for continual, partner-specific learning even among mature language users with existing priors.
This raises another question: how are these different timescales of learning related to one another?
One possibility is that the need to quickly adapt one's language to new partners and contexts over short timescales may serve as a functional pressure shaping languages more broadly.

Recent theories in cognitive science have formalized this hypothesis in a \emph{hierarchical} Bayesian model \cite{hawkins2020generalizing}.
In this model, the prior $\Theta$ that an agent brings into subsequent interactions is updated to reflect the overall distribution of partner-specific models $\theta^i$, thus balancing general and idiosyncratic language knowledge in a principled way.
For neural language models, however, there is an apparent tension between the strong KL regularization required to \emph{prevent} unwanted interference with background knowledge during partner-specific adaptation, leading to catastrophic forgetting, and the flexibility to generalize or transfer conventions to new communicative settings as required for language learning.
We do not want to regularize so strongly that agents memorize conventions only applying to a single image that is completely reset after each interaction; instead, we wish to obtain a gradient of generalization across both referents and partners as a function of similarity \cite{markman1998referential}.

One promising solution to this problem, motivated by connections between hierarchical Bayes and algorithms like MAML \cite{finn2017model,grant_recasting_2018,nagabandi_deep_2018}, is to perform a meta-learning `outer loop' updating the initialization $\Theta$, taking into account the regularized, partner-specific `inner loop' of adaptation for each $\theta^i$.
In principle, a meta-learning approach for neural language learning would distill abstract, shared aspects of language into a unified $\Theta$, while still allowing for rapid \emph{ad hoc} conventionalization.
Still, cognitively plausible and scalable meta-learning algorithms remain an open area of research.


\paragraph{Limitations and future work}

While our evaluations were limited to a canonical CNN-RNN image captioning architecture, a key open question for future work is how our continual adaptation approach ought to be implemented for more complex, state-of-the-art architectures. 
One possibility, following the approach recently proposed by \citet{jaech_low-rank_2017}, is to allow context (e.g. partner identity) to control a low-rank transformation of the weight matrix such that online fine-tuning can take place in a more compact context embedding space \cite{jaech_personalized_2018}.

Furthermore, while we adapted the entire parameterized RNN module end-to-end, future work should explore the effect of limiting adaption to subcomponents (e.g.\ only word embeddings) or expanding adaptation to additional model components such as attention weights or high-level visual representations.
Beyond possible consequences for engineering better adaptive models, each of these variants corresponds to a distinct cognitive hypothesis about exactly \emph{which} representations are being adapted on the fly in human communication.

A final area for future work is generalizing the forms of social feedback that can be used as data $D^i$ for updating representations beyond the sparse choices in a reference game.
In particular, forms of \emph{repair} through question-asking or other non-referential dialogue acts may license stronger inferences about a partner's language model and allow misunderstandings to be resolved more quickly in challenging contexts \cite{drew1997open,DingemanseEtAl15_RepairUniversal,li2016learning}.
These forms of feedback may be particularly important for extending our approach beyond the benchmark task of repeated reference games to the more complex domains of real-world conversational tasks.

\paragraph{Conclusion}

Human language use is remarkably flexible, continuously adapting to the needs of the current situation.
In this paper, we introduced a challenging repeated reference game benchmark for artificial agents, which requires such adaptability to succeed.
We proposed a continual learning approach allowing agents to form context-specific conventions by fine-tuning general-purpose representations.
Even when pre-trained models initially perform inaccurately or inefficiently, our approach allows such models to quickly adapt to their partner's language in the given context and thus become more accurate and more efficient using common ground.

\section*{Acknowledgments}

This research was supported in part by a Stanford HAI Hoffman-Yee Research Grant, Office of Naval Research grant ONR MURI N00014-16-1-2007 and DARPA agreement FA8650-19-C-7923, as well as NSF award \#1911835 to RDH, and NSF award \#1941722 to DS.
We are grateful to audiences at the 2019 ICML Workshop on Adaptive and Multi-Task Learning, where an early version of this work was presented, and to three anonymous reviewers for their insightful comments.

\vspace{2em}
\fbox{\parbox[b][][c]{6.8cm}{\centering {All code and materials available at: \\
\href{https://github.com/hawkrobe/continual-adaptation}{\url{https://github.com/hawkrobe/continual-adaptation}}
}}}
\vspace{2em} \noindent

\bibliographystyle{acl_natbib}
\bibliography{../references.bib}

\renewcommand{\thefigure}{S\arabic{figure}}
\renewcommand{\thetable}{S\arabic{table}}

\newpage

\begin{figure*}[t]
\centering
\includegraphics[scale=0.9]{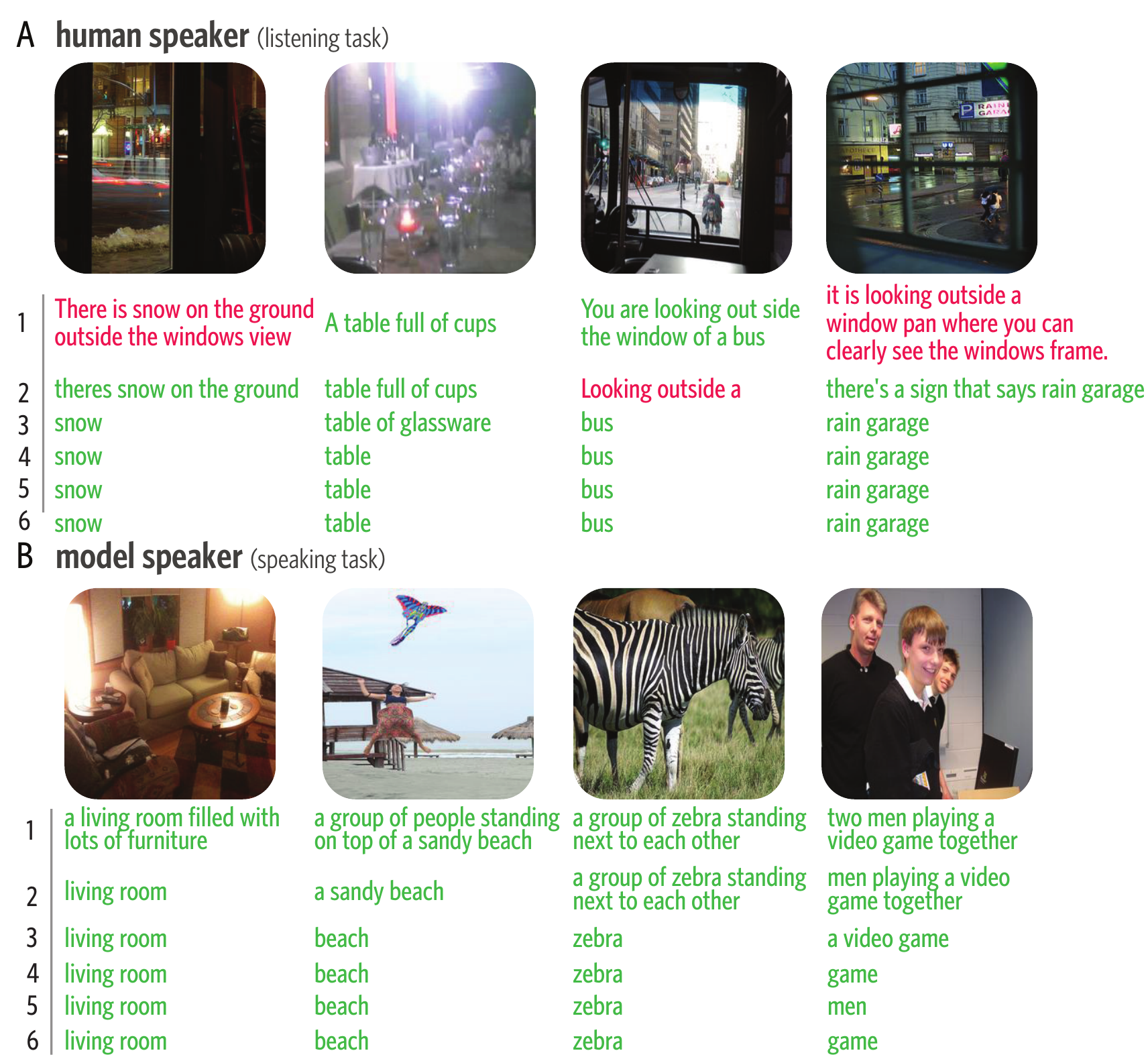}
\caption{Complete set of referring expressions produced by (A) a \emph{human speaker} interacting with our listener model and (B) our \emph{speaker model} interacting with a human partner, as described below in Appendix D. Utterances are color-coded with the response accuracy. Green is correct; red is incorrect.}
\label{fig:examples}
\end{figure*}

\section*{Appendix A: derivation of incremental KL}

We denote the distribution over a sequence of $T$ tokens by $p(w_{1:T}) = p(w_1, w_2, \dots, w_T)$.
We are interested in the KL divergence between two such distributions, $\KL{p(w_{1:T})}{q(w_{1:T})}$.
We show that our approximation over possible captions is the best incremental estimator of this intractable objective. 
First, note that the KL divergence factors in the following way.

\begin{lemma}
\begin{equation*}
\begin{split}
& \KL{p(w_{1:2})}{q(w_{1:2})}  \\
& = \, \KL{p(w_1)}{q(w_1)} \\
 & \quad + \mathbb{E}_{p(w_1)} \KL{p(w_2 | w_1)}{q(w_2 | w_1)}\\
 \end{split}
\end{equation*}
\end{lemma}
\begin{proof}
\phantom\qedhere
\begin{equation*}
\begin{split}
& \KL{p(w_{1:2})}{q(w_{1:2})} \\
& = \sum_{w_1}\sum_{w_2}p(w_{1:2})\log\frac{p(w_{1:2})}{q(w_{1:2})} \\
& =   \sum_{w_1}\sum_{w_2}p(w_{1:2})\log\frac{p(w_1)}{q(w_1)}  \\
 &\quad + \sum_{w_1}\sum_{w_2}p(w_{1:2})\log\frac{p(w_2|w_1)}{q(w_2|w_1)} \\
 & = \sum_{w_1}\log\frac{p(w_1)}{q(w_1)}\sum_{w_2}p(w_{1:2}) \\
 &\quad + \sum_{w_1}p(w_1)\sum_{w_2}p(w_2 | w_1)\log\frac{p(w_2|w_1)}{q(w_2|w_1)} \\
 & = \KL{p(w_1)}{q(w_1)} \\
 & \quad + \mathbb{E}_{p(w_1)} \KL{p(w_2 | w_1)}{q(w_2 | w_1)}
\end{split}
\end{equation*}
\end{proof}

Now, let $w^{*}_1$ be the token at which $p(w_1)$ takes its maximum value.
Then $w^{*}_1$ is the best single-sample approximation of the expectation:
\begin{equation*}
\begin{split}
& \mathbb{E}_{p(w_1)} \KL{p(w_2 | w_1)}{q(w_2 | w_1)} \\
& \approx \KL{p(w_2 | w^{*}_1)}{q(w_2|w^{*}_1}\\
\end{split}
\end{equation*}
If we assume that $p(w_{1:T})$ is Markov (as in a recurrent model) then it follows from repeatedly applying the lemma that
\begin{equation*}
\begin{split}
& \KL{p(w_{1:T})}{q(w_{1:T})} \\
& = \sum_{i=1}^T \KL{p(w_i | w^{*}_1, \dots, w^*_{i-1})}{q(w_i|w^*_1\dots, w^*_{i-1}}\\
& = \sum_{i=1}^T \KL{p(w_i | w^*_{i-1})}{q(w_i | w^*_{i-1})}
\end{split}
\end{equation*}
recovering our objective.


\section*{Appendix B: Parameter settings}

For both the speaker task and listener task, we used a learning rate of 0.0005, took 6 gradient steps after each trial, and used a batch size of 8 when sampling utterances from the augmented set of sub-phrases, 
At each gradient step, we sampled 50 objects from the full domain $\mathcal{O}$ of COCO to approximate our regularization term.
We set the coefficients weighting each term in our loss function as follows: 1.0 (utterance loss), 0.1 (contrastive loss), 0.5 (KL regularization), 0.3 (local rehearsal).

 \begin{figure*}[t]
\centering
\includegraphics[scale=0.95]{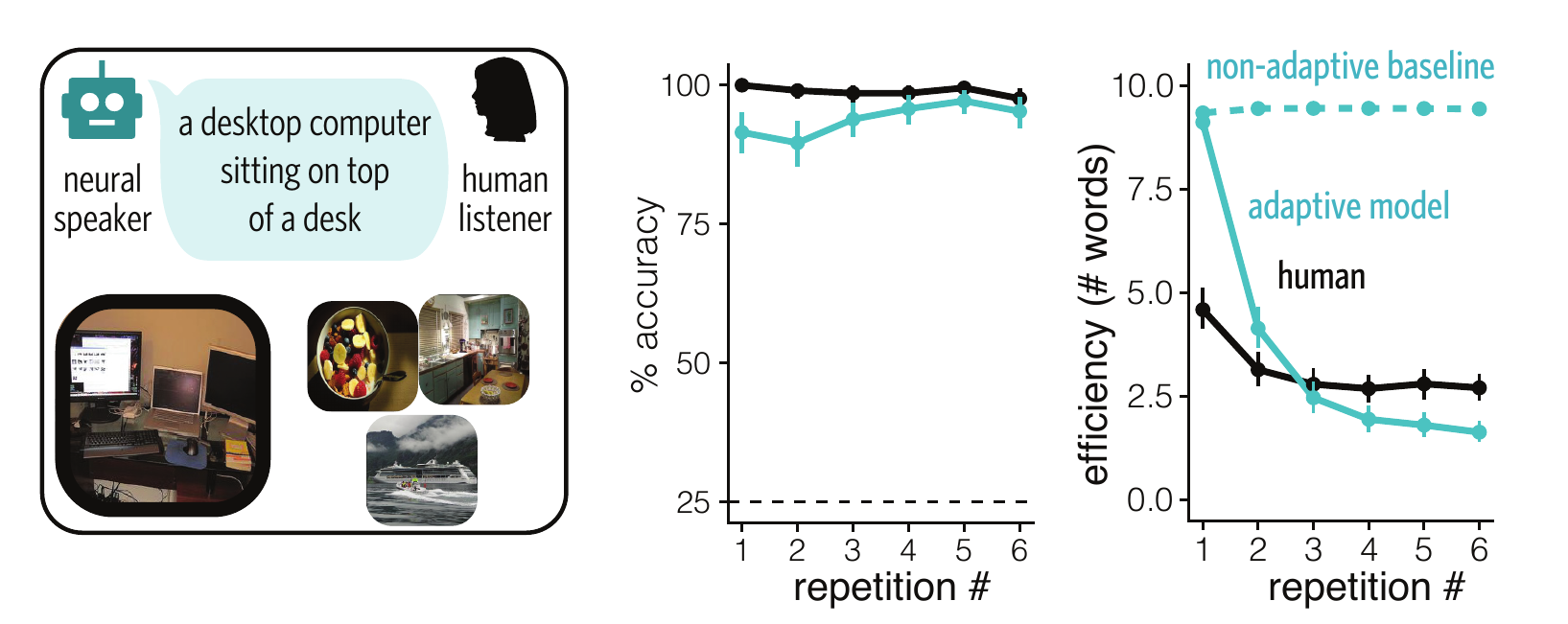}
\caption{Speaker model evaluations with human listener. Error ribbons are bootstrapped 95\% CIs.}
\label{fig:supplementalcurves}
\end{figure*}

\section*{Appendix C: Regression details}

To formally test increasess in efficiency reported for baseline pairs of humans in Sec.~4.1 (see Fig. 3 and \ref{fig:supplementalcurves}), we conducted a mixed-effects regression predicting utterance length. 
We included a fixed effect for context type (i.e. `simple' vs. `challenging') as well as orthogonalized linear and quadratic effects of repetition number, and each of their interactions with context type. 
We also included random intercepts accounting for variability in initial utterance length at the pair- and image-level.
To test increases in the adaptive listener model's accuracy in Sec.\ 4.2, we conducted a mixed-effects logistic regression on trial-level responses (i.e. `correct' vs. `incorrect') with the same effect structure, removing context type, as this experiment was only conducted on challenging contexts.
This same effect structure was used to test improvements in the adaptive speaker model's efficiency in Sec.~4.3. 

We tested our listener ablations in Sec.\ 5.2 using a mixed-effects logistic regression with fixed effects of repetition number and model variant, as well as participant-level random intercepts and slopes (for repetition number). We dummy-coded the models setting the baseline at the full model, to facilitate direct comparison.

\section*{Appendix D: Speaker evaluation and analysis}

\subsection*{Task description}

We designed the speaking task such that the pre-trained model's \emph{efficiency} --- the number of words needed to identify the target --- would be poor at the outset. 
Because the COCO captions seen during pre-training were relatively exhaustive (i.e.~mentioning many attributes of the image), we required \emph{simple} contexts where the pre-trained model would produce more complex referring expressions than required to distinguish the images. 
To construct simple contexts we sampled images randomly from different COCO category labels.
For example, one context might contain an image of an elephant, an image of a boat, and so on.

\subsection*{Evaluation results}
\label{sec:speaking_task}

We evaluated our model in the \emph{speaking} task using simple contexts, which requires the model to form more efficient conventions given feedback from human responses.
53 participants from Amazon Mechanical Turk were paired to play the listener role with our speaker model.
Utterances were selected from the LSTM decoder using beam search with a beam width of 50 and standard length normalization to mitigate the default bias against long utterances \cite[e.g.][]{wu2016google}.
After producing an utterance, the model received feedback about the listener's selection. 
If its partner correctly selected the intended target, it proceeded to adapt conditioning on the new observation; in the event of an incorrect response, it refrained from updating.
This strategy thus only leads to inferences about utterance meanings (and sub-phrase meanings, through data augmentation) after positive evidence of understanding.

As expected, the model starts with much longer captions than human speakers use in simple contexts (Fig.\ \ref{fig:supplementalcurves}).
It uses nearly as many words for simple contexts as humans used for challenging contexts. 
However, it gets dramatically more efficient over interaction while maintaining high accuracy. 
We found a significant decrease in utterance length over successive repetitions, $t=35,~p <0.001$, using the same mixed-effects regression structure reported above.
A non-adapting baseline shows no improvement, as it has no mechanism for changing its expectations about utterances over time.

\begin{table*}[h!]
\resizebox{\textwidth}{!}{%
\begin{tabular}{p{4cm} | p{3.5cm} | p{3cm} | p{2.7cm} | p{1.6cm} | p{1.3cm}}
Rep.~1 & Rep.~2 & Rep.~3 & Rep.~4 & Rep.~5 & Rep.~6 \\
\hline
a group of people standing on a sandy beach & a group of people standing on top        & a group of people standing          & a group of people & a group of & a group  \\
\hline
a couple of zebra standing next to trees    & a couple of zebra standing next          & a couple of zebra                   & a couple of       & a couple   & a couple \\
\hline
a living room filled with lots of furniture & a living room filled with lots furniture & a living room filled with furniture & a living room     & a living   & a living
\end{tabular}%
}
\caption{Examples of utterances produced by ablated speaker with pure cost penalty instead of data augmentation, which quickly become ungrammatical and incoherent.}
\label{tab:badbad_examples}
\end{table*}

\subsection*{Pragmatic reasoning supports speaker informativity}
\label{sec:pragmatic_analysis}

We now proceed to analyze the \emph{speaking task}, beginning with the role of pragmatic reasoning.
In principle, incorporating pragmatic reasoning during adaptation (i.e. in our contrastive likelihood term) introduces an inductive bias for \emph{mutual exclusivity} \cite{SmithGoodmanFrank13_RecursivePragmaticReasoningNIPS, FrankGoodmanTenenbaum09_Wurwur,gandhi2019mutual}.
When the listener correctly selects the target, the speaker not only learns that the listener believes this is a good description for the target but can also infer that the listener \emph{does not} think it is a good description for the other objects in context; otherwise, they would have selected one of the other objects instead.
Thus, in addition to boosting listener adaptation, we expected explicit pragmatic reasoning to allow the speaker to gradually produce more informative, distinguishing utterances. 

The \emph{challenging} contexts provide an ideal setting for evaluating speaker pragmatics, because the pre-trained speaker model initially produces the same caption for all four images in context.
We simulated games with our adaptive speaker as well as an ablated variant with no contrastive term in its adaptation objective.
The model was always given feedback that the correct target was selected.
We measured informativity by examining the proportion of words that overlapped between the utterances produced for the different images in a particular context: $|u_i \cap u_j| / \min(|u_i|, |u_j|)$ for combinations of utterances ($u_i,u_j$) where $i \neq j$. 
This measure ranges between 0\% when the intersection is empty and 100\% in the event of total overlap. 

Even though both the full model and the ablated variant initially produce completely overlapping utterances, we found a rapid differentiation of utterances for the model with pragmatics intact, as each image becomes associated with a distinct label.
Meanwhile, the ablated version continues to have high overlap even on later repetitions: it often ends up producing the same one-word label for multiple objects as it reduces (Fig.\ \ref{fig:speakerablation}A).

 \begin{figure}[b!]
\centering
\includegraphics[scale=0.7]{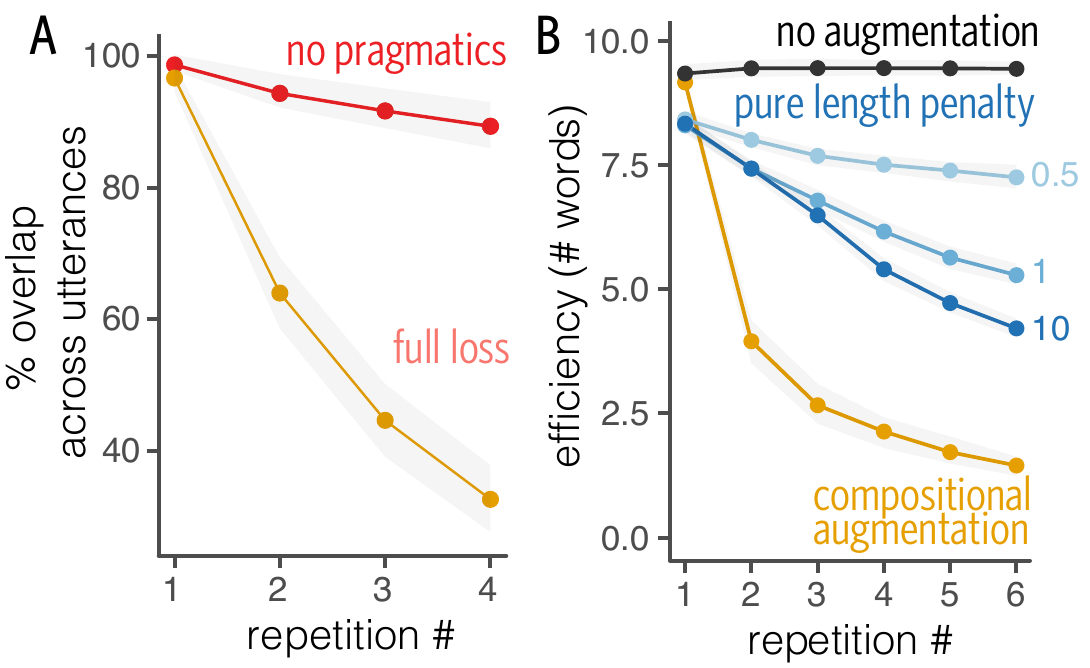}
\vspace{-2em}
\caption{Speaker model ablations. (A) The contrastive loss allows the model to become informative in challenging contexts. (B) Compositionally augmenting adaptation data with sub-phrases of the utterance allows stronger gains in efficiency than a simple length penalty. Error ribbons are bootstrapped 95\% CIs.}
\vspace{-.5em}
\label{fig:speakerablation}
\end{figure}

\subsection*{Compositional data augmentation supports efficiency}
Finally, we investigated the role played by the compositional data augmentation mechanism for allowing our speaker model to become more efficient (Fig.\ \ref{fig:speakerablation}B).
Two concerns may be raised about this mechanism.
First, it is possible that the RNN decoder architecture is already able to appropriately update expectations about sub-parts from the whole without being given explicit parses, so augmentation is redundant. 
Second, it may be argued that this augmentation mechanism just imposes a glorified length penalty \emph{forcing} the speaker to shorten, rather than allowing efficiency to come out of the model naturally.

To address these concerns, we compare augmentation with two variants: (1) an ablated model with no augmentation, and (2) an alternative mechanism that explicitly imposes a length cost at production time. 
This alternative is implemented by re-ranking the top 25 utterances from beam search according to $U(u_i) = P(u_i)/\ell(w) - \beta_w \ell(w)$ where the first term is the length-normalized beam-search objective and the second term is an explicit bias for shorter utterances.
When $\beta_w=0$, this is equivalent to top-$k$ beam search but as $\beta_k \rightarrow \infty$, the model will increasingly prefer short utterances.

We simulated the behavior of these model variants in each of the 53 games we collected in our interactive speaking task, using the same sequence of images and feedback about correctness. 
We found that that the ablated model with no augmentation fails to become more efficient: at least for our RNN decoder architecture, evidence of success only reinforces expectations about the full caption; it cannot not propagate this evidence to the individual parts of the utterance.
A sufficiently high length penalty does allow utterances to become shorter, but reduces linearly rather than quadratically (as humans do; \citealp{ClarkWilkesGibbs86_ReferringCollaborative,hawkins2019characterizing}).
Moreover, upon inspecting the utterances produced by this variant (see Table \ref{tab:badbad_examples}), we found that it simply cuts off the ends of utterances, whereas compositional data augmentation is based on a syntactic parse and thus allows the model to preserve grammaticality and gradually build expectations about meaningful sub-units.
In sum, we find that our augmentation mechanism is not reducible to a coarse length bias, and is able to compensate for representation failures in current recurrent architectures \cite{dasgupta2018evaluating, nikolaus2019compositional}.\footnote{We expect that a more structured architecture, able to propagate evidence about the meaning of a full utterance to representations of intermediate semantic units, would make this augmentation step redundant \cite[e.g.][]{tai2015improved,gan2017semantic,mccoy_does_2020}.}

\end{document}